\documentclass[lettersize,journal]{IEEEtran}
\IEEEoverridecommandlockouts

\usepackage{cite}
\usepackage{amsmath,amssymb,amsfonts}
\usepackage{graphicx}
\usepackage{subfigure}
\usepackage{threeparttable}
\usepackage{setspace}
\usepackage{makecell}
\usepackage[table,xcdraw]{xcolor}
\usepackage{tabularx}
\usepackage{textcomp}
\usepackage{xcolor}
\usepackage{stfloats}
\usepackage{amsfonts}
\usepackage{amssymb}
\usepackage{amsthm}
\usepackage{cite}
\usepackage{float}
\usepackage{color}
\usepackage{stfloats,fancyhdr}
\usepackage{amsmath,bm}
\usepackage{algorithm}
\usepackage{changepage}
\usepackage[normalem]{ulem}
\usepackage{balance}
\usepackage{graphicx}
\usepackage{bbm}
\usepackage{graphicx}
\usepackage{multirow}

\usepackage{threeparttable}
\usepackage{algpseudocode}
\usepackage{setspace}

\newtheorem{problem}{Problem}
\newtheorem{theorem}{Theorem}
\newtheorem{lemma}{Lemma}

\newtheorem{remark}{Remark}
\newtheorem{example}{Example}

\def\BibTeX{{\rm B\kern-.05em{\sc n\kern-.025em b}\kern-.08em
   T\kern-.1667em\lower.7ex\hbox{E}\kern-.125emX}}

\usepackage{mathtools}
\usepackage{comment}

\title{\huge Semantic-aware Transmission Scheduling: a Monotonicity-driven Deep Reinforcement Learning Approach}
\vspace{-1cm}
\author{\large Jiazheng Chen,{\IEEEmembership{Graduate Student Member,~IEEE,}}, Wanchun Liu*,~\IEEEmembership{Member,~IEEE,} Daniel~E.~Quevedo,~\IEEEmembership{Fellow,~IEEE,} Yonghui~Li,~\IEEEmembership{Fellow,~IEEE,} Branka Vucetic,~\IEEEmembership{Life Fellow,~IEEE}
\thanks{J. Chen, W. Liu, Y. Li, and B. Vucetic are with the School of Electrical and Information Engineering, The University of Sydney, Sydney, NSW 2006, Australia (e-mails: \{jiazheng.chen, wanchun.liu, yonghui.li, branka.vucetic\}@sydney.edu.au).
D. E. Quevedo is with the School of Electrical Engineering and Robotics, Queensland University of Technology (QUT), Brisbane, Australia. (e-mail: dquevedo@ieee.org).
\textit{(W. Liu is the corresponding author.)}}
\vspace{-1.1cm}
} 

\begin{document}
\maketitle

\vspace{-0.6cm}
\begin{abstract}
For cyber-physical systems in the 6G era, semantic \\ communications connecting distributed devices for dynamic control and remote state estimation are required to guarantee application-level performance, not merely focus on communication-centric performance. Semantics here is a measure of the usefulness of information transmissions. Semantic-aware transmission scheduling of a large system often involves a large decision-making space, and the optimal policy cannot be obtained by existing algorithms effectively. In this paper, we first establish the monotonicity of the Q function of the optimal semantic-aware scheduling policy and then develop advanced deep reinforcement learning (DRL) algorithms by leveraging the theoretical guidelines. Our numerical results show a 30\% performance improvement compared to benchmark algorithms.
\end{abstract}

\begin{IEEEkeywords}
Deep reinforcement learning, sensor scheduling, age of information, semantic communications.
\end{IEEEkeywords}
\vspace{-0.5cm}

\section{Introduction} \label{sec:intro}
Different from the first generation to the 5G communications, which focused on improving data-oriented
performance (e.g., data rate, latency and reliability), the 6G is envisioned to support many emerging applications and adopt semantic-related metrics for application-level performance optimization~\cite{gunduz2022beyond}.
The potential applications supported by 6G semantic communications can be divided into two categories: human-centered systems (e.g., Metaverse and XR~\cite{chen2023aoii}) and cyber-physical systems (CPSs) (e.g., autonomous driving and networked robotics~\cite{kountouris2021semantics}). For the former, semantics is a measure of the meaning of information that a human can understand; for the latter, semantics characterizes the usefulness of messages with respect to the goal of data exchange.

In CPSs, semantic-aware transmission aims to minimize the distortion and timing mismatch of the data between the edge devices and the remote destination~\cite{kountouris2021semantics}.
Thus, the freshness of messages is critical, and monotonic functions of age of information (AoI) are often introduced as semantic quality metrics~\cite{wu2018optimalmulti,liu2021remoteMF}.
Semantic-aware (and AoI-aware) dynamic scheduling of cyber-physical systems over limited communication resources has attracted much attention in the communications society~\cite{gunduz2022beyond}.
In~\cite{wu2018optimalmulti}, an optimal scheduling problem of a multi-sensor remote estimation system was considered to minimize the overall long-term estimation mean-square error (MSE). Importantly, the instantaneous MSE of each sensor is a function of the corresponding AoI. The scheduling problem was formulated as a Markov decision process (MDP) that can be solved by conventional value and policy iteration methods.
In~\cite{tang2020minimizing}, an optimal AoI minimization scheduling problem was considered under an average transmission power constraint and was formulated into a constrained MDP.
Note that the conventional MDP algorithms can only solve very small-scale problems due to the high computational complexity introduced by the curse of dimensionality. 
Heuristic scheduling policies without optimality guarantees were proposed in were proposed in~\cite{qian2020minimizing,tong2022age}.

The recent development of deep reinforcement learning (DRL) allows us to solve large MDPs effectively with deep neural networks (NNs) for function approximations.
In~\cite{Yang2022OMA}, deep Q-network (DQN) was adopted to solve optimal scheduling problems of multi-sensor-multi-channel remote estimation systems. In particular, those solutions can achieve better performance (in terms of MSE) than heuristic methods.
Advanced DRL algorithms with an actor-critic framework (consisting of an actor NN and a critic NN), such as soft actor-critic (SAC) and deep deterministic policy gradient (DDPG), were introduced to deal with large-scale semantic-aware scheduling problems that cannot be handled by DQN~\cite{chen2022seDRL,akbari2021age}.
\emph{A major drawback of applying existing DRL methods in large-scale scheduling problems is that training is usually very time-consuming and can often get stuck in local minima.
This motivates us to have a more profound  understanding of the structural properties of optimal scheduling policies (such as the monotonicity of the critic NN), which can then be leveraged to guide DRL algorithms towards faster and more effective convergence.}

In this paper, we focus on \textbf{developing advanced actor-critic DRL algorithms to optimize semantic-aware scheduling policies}. The contributions include: First, we mathematically prove that a well-trained critic NN of an optimal scheduling policy should be a \textbf{monotonic function} in terms of the AoI states and channel states of the system. Such a property has never been established in the open literature.
Second, by leveraging this theoretical property, we propose two monotonic critic NN training methods: \textbf{network architecture-enabled monotonicity} and \textbf{regularization-based monotonicity}. The proposed training methods can be directly applied to the baseline DDPG and its variants.
Last, our numerical results show that the proposed \textbf{monotonicity-enforced} DDPG algorithms can save training time and improve training performance significantly compared to the baseline DDPG.

\vspace{-0.3cm}
\section{System Model} \label{sec: sys}
We consider semantic communication in a CPS with $N$ edge devices (e.g., sensors and robots) and a remote destination (e.g., a remote estimator or controller) connected by a wireless network with $M<N$ channels (e.g., sub-carriers).

\subsection{Wireless Communication Model}
At each time step $t$, the transmission scheduling action for device $n$ is denoted by\vspace{-0.1cm}
\begin{equation}\label{eq:action}
    a_{n,t} = \left\{
    \begin{array}{ll}
        0  & \text{if device $n$ is not scheduled,} \\
        m  & \text{if device $n$ is scheduled to channel $m$.}
    \end{array}
    \right.\vspace{-0.1cm}
\end{equation}
We assume that each device can be scheduled to at most one channel and that each channel is assigned to one device, i.e.,\vspace{-0.1cm}
\begin{equation}\label{eq: action constraint}
    \sum_{m=1}^{M} \boldsymbol{\mathbbm{1}}\left( a_{n,t} = m \right) \leq 1, \quad 
    \sum_{n=1}^{N} \boldsymbol{\mathbbm{1}}\left( a_{n,t} = m \right) = 1,\vspace{-0.1cm}
\end{equation} where $\boldsymbol{\mathbbm{1}} (\cdot)$ is the indicator function.

We consider independent and identically distributed (i.i.d.) block fading channels, where the channel quality is fixed during each packet transmission and varies packet by packet independently.
The channel state of the system at time $t$ is denoted by a $N \times M$ matrix $\mathbf{H}_{t}$, where the $n$th row and $m$th column element  $h_{n,m,t} \in \mathcal{H} \triangleq \left\{ 1, 2, \dots, \hat{h} \right\}$ represents the channel state between device $n$ and the destination at channel~$m$ quantized into $\hat{h}$ levels.
The distribution of the channel state $h_{n,m,t}$ is given~as\vspace{-0.2cm}
\begin{equation}\label{eq:q}
\operatorname{Pr}(h_{n,m,t}=j) = q^{(n,m)}_{j}, \forall t,\vspace{-0.1cm}
\end{equation}
where $\sum_{j=1}^{\hat{h}} q^{(n,m)}_{j} =1, \forall n,m$.
A higher channel state leads to a large packet drop probability.
Let $p_{n,m,t}$ denote the packet drop probability at the channel state $h_{n,m,t}$.
The instantaneous full channel state $\mathbf{H}_{t}$ is available at the remote estimator based on standard channel estimation methods.
\vspace{-0.2cm}

\subsection{Semantic Communication Performance}
First, we define $\tau_{n,t}\in \{1,2,\dots\}$ as the AoI of device $n$ at time $t$, i.e., the time elapsed since its latest device packet was successfully received~\cite{liu2021remoteMF, tang2020minimizing, qian2020minimizing}:\vspace{-0.1cm}
\begin{equation}\label{eq:tau}
\!\!\!\!\!\tau_{n,t+1} \!=\!\!\begin{cases}
\!1, & {\!\!\!\!\!\text{at time $t$, device $n$'s packet is received,}}\\
\!\tau_{n,t}\!+\!1, & {\!\!\!\!\text{otherwise.}}
\end{cases}\!\!\!\!\!\vspace{-0.1cm}
\end{equation}

Next, we define a \emph{positive non-decreasing cost} (or penalty) function $g_n(\tau_{n,t}), \forall n \in \{1,\dots,N\}$ as a semantic measurement of device~$n$'s information. A higher AoI state corresponds to an increased cost function, indicating poorer performance in semantic communications. Semantic communication systems with varying functionalities possess distinct \(g_n(\cdot)\) functions. The scheduling algorithms proposed in the rest of this paper are general and applicable to any positive and non-decreasing cost function.

An example below shows how to determine $g_n(\cdot)$ in a specific semantic communication system.

\begin{example}[A remote state estimation system]\label{eg:1}
Such a system consists of $N$ sensors, each measuring a physical process, and a remote estimator for signal reconstruction.
Each physical process dynamics and its measurement are modeled as a discrete-time linear time-invariant (LTI) system:\vspace{-0.1cm}
\begin{equation}\label{eq:LTI}
\begin{aligned}
    \mathbf{e}_{n,t+1}  = \mathbf{A}_{n} \mathbf{e}_{n, t} + \mathbf{w}_{n, t},  \ 
    \mathbf{z}_{n, t}  = \mathbf{C}_{n} \mathbf{z}_{n, t} + \mathbf{v}_{n, t},\vspace{-0.1cm}
\end{aligned}
\end{equation}
where $\mathbf{e}_{n, t} \in \mathbb{R}^{l_{n}}$ is  process $n$'s state at time $t$, and  $\mathbf{z}_{n, t} \in \mathbb{R}^{c_{n}}$ is the state measurement of sensor $n$; 
$\mathbf{A}_{n} \in \mathbb{R}^{l_{n} \times l_{n}}$ and $\mathbf{C}_{n} \in \mathbb{R}^{c_{n} \times l_{n}}$ are the system matrix
and the measurement matrix, respectively; $\mathbf{w}_{n, t} \in \mathbb{R}^{l_{n}}$ and $\mathbf{v}_{n, t} \in \mathbb{R}^{c_{n}}$ are the process disturbance and the measurement noise modeled as i.i.d zero-mean Gaussian random vectors $\mathcal{N}(\mathbf{0},\mathbf{W}_{n})$ and $\mathcal{N}(\mathbf{0},\mathbf{V}_{n})$, respectively. 
Due to the noisy measurement in~\eqref{eq:LTI}, sensor $n$ applies a local Kalman filter (KF) to pre-process the raw measurement before sending it to the remote estimator.
Due to the dynamic scheduling and the random packet dropouts, the remote estimator cannot always receive sensor~$n$'s measurement and thus uses a minimum mean square error estimator for state estimation, say $\hat{\mathbf{e}}_{n,t}$. Then, the estimation error covariance is given as \vspace{-0.1cm}
\begin{align}
  \mathbf{P}_{n, t} & \triangleq \mathbb{E} \left[ \left(\hat{\mathbf{e}}_{n, t} - \mathbf{e}_{n, t}\right) \left(\hat{\mathbf{e}}_{n, t} - \mathbf{e}_{n, t} \right)^{\top} \right] 
  = f_{n}^{\tau _{n,t}}(\bar{\mathbf{P}}_{n}),\vspace{-0.1cm}
\end{align}
where $f_{n}(\mathbf{X}) = \mathbf{A}_{n} \mathbf{X} \mathbf{A}_{n}^{\top} + \mathbf{W}_{n}$ and $f^{\tau+1}_{n}(\cdot) =f_n(f^{\tau}_{n}(\cdot))$, where $\bar{\mathbf{P}}_{n}$ is a constant matrix determined by $\mathbf{A}_n$, $\mathbf{C}_n$, $\mathbf{W}_n$ and $\mathbf{V}_n$.

{To characterize the remote estimation performance, we define the cost function as the estimation MSE of process $n$  \vspace{-0.1cm}
\begin{equation}
    g_n(\tau_{n,t}) \triangleq \operatorname{Tr}(\mathbf{P}_{n,t}) = f_{n}^{\tau _{n,t}}(\bar{\mathbf{P}}_{n}), \vspace{-0.1cm}
\end{equation}
which has been proved to be a positive non-decreasing function of $\tau_{n,t}$, see~\cite{liu2021remoteMF} for further details. 
}

\end{example}
\vspace{-0.3cm}
\section{Semantic-aware Transmission Scheduling} \label{sec: problem}
We aim to find a dynamic scheduling policy $\pi(\cdot)$ that minimizes the expected total discounted cost function of all $N$ devices over the infinite time horizon.
\begin{problem}\label{pro1}\vspace{-0.3cm}
\begin{equation}
    \min_{\pi} \lim_{T \to \infty} \mathbb{E}^\pi \left[ \sum_{t=1}^{T} \sum_{n=1}^N \gamma^t g_n(\tau_{n,t}) \right],\vspace{-0.1cm}
\end{equation}
where $\gamma\in(0,1)$ is a discount factor.
\end{problem}
\vspace{-0.2cm}

\subsection{Markov Decision Process Formulation}
Problem~\ref{pro1} is a sequential decision-making problem with the Markovian property and hence can be cast as an MDP as below. 

1) The state of the MDP consists of both the AoI and channel states as $\mathbf{s}_{t} \triangleq \left( \bm{\tau}_{t}, \mathbf{H}_{t} \right) \in \mathbb{N}^N \times \mathcal{H}^{M \times N}$, where $\bm{\tau}_{t} = (\tau_{1,t}, \tau_{2,t}, \dots, \tau_{N,t}) \in \mathbb{N}^N$ is the AoI state vector.
    
2) The overall schedule action of the $N$ devices is defined as $\mathbf{a}_t = (a_{1,t}, a_{2,t}, \dots, a_{N,t}) \in \left\{ 0, 1, 2, \dots, M \right\}^N$ under the constraint~\eqref{eq: action constraint}.
The policy $\pi(\cdot)$ maps a state to an action as $\mathbf{a}_{t} = \pi(\mathbf{s}_{t}) $.

3) The immediate reward (i.e., the negative sum cost of all edge devices) at time $t$ is defined as $-\sum_{n=1}^N g_n(\tau_{n,t})$.

4) The transition probability $\operatorname{Pr}(\mathbf{s}_{t+1}|\mathbf{s}_t, \mathbf{a}_t)$ is the probability of the next state $\mathbf{s}_{t+1}$ given the current state $\mathbf{s}_{t}$ and the  action $\mathbf{a}_t$. 
Since the optimal policy of an infinite-horizon MDP is stationary~\cite[Chapter~6]{puterman2014markov}, the state transition is independent of the time index. Thus, we drop the subscript $t$ and use $\mathbf{s}\triangleq (\bm\tau,\mathbf{H})$ and $\mathbf{s}^+ \triangleq (\bm\tau^+,\mathbf{H}^+)$ to represent the current and the next states, respectively.
Due to the i.i.d. fading channels and the AoI state transition, we have \vspace{-0.2cm}
\begin{equation}\label{eq: transition p}
   \operatorname{Pr}(\mathbf{s}^+|\mathbf{s}, \mathbf{a}) = \operatorname{Pr}(\bm{\tau}^+|\bm{\tau},\mathbf{H}, \mathbf{a}) \operatorname{Pr}(\mathbf{H}^+), 
\end{equation}
where $\operatorname{Pr}(\bm{\tau}^+|\bm{\tau},\mathbf{H}, \mathbf{a}) = \prod_{n=1}^N \operatorname{Pr}(\tau^+_{n}|\tau_{n}, \mathbf{H}, a_{n})$ and\vspace{-0.1cm}
\begin{align}\label{eq: transition p 2}
    \!\!\!\!\!\operatorname{Pr}(\tau^+_{n}|\tau_{n}, \mathbf{H}, a_{n})  \!=\! 
    \begin{cases}
        1 - p_{n, m}, & \!\!\!\text{if $\tau^+_{n} = 1, a_{n} = m$,} \\
        p_{n, m}, & \!\!\!\text{if $\tau^+_{n} = \tau_{n} \!+\! 1, a_{n} \!= \!m$,}\\
        1,          & \!\!\!\text{if $\tau^+_{n} = \tau_{n} \!+\! 1, a_{n} \!=\! 0$,} \\
        0, & \!\!\!{\text{otherwise.}}
    \end{cases}\vspace{-0.3cm}\!\!\!\!
\end{align}

\subsection{Value Functions of the Optimal Policy}
We define the V function (also called  ``state-value function"), $V(\mathbf{s}_{t}): \mathcal{S} \to \mathbb{R}$ and the Q function (the ``action-value function"), $Q(\mathbf{s}_{t}, \mathbf{a}_{t}): \mathcal{S} \times \mathcal{A} \to \mathbb{R}$ of the optimal policy $\pi^*(\cdot)$.
Given the current state $\mathbf{s}_t$, the V function is the maximum expected discounted sum of the future reward achieved by the optimal policy, i.e.,
$    V(\mathbf{s}_{t}) = \mathbb{E}\left[ \sum_{t'=t}^{\infty} \gamma^{t'-t} r\left(\mathbf{s}_{t'}\right) \Big\vert \mathbf{a}_{t'} = \pi^{*}(\mathbf{s}_{t'}),\forall t'\geq t\right].
$
Given the current state-action pair, $(\mathbf{s}_t,\mathbf{a}_{t})$, the Q function is the maximum expected discounted sum of the future reward under the optimal policy:
\begin{equation}\label{eq: def Q}
    \!Q(\mathbf{s}_{t}, \mathbf{a}_{t}) \!=\!  \mathbb{E} \left[\sum_{t'=t}^{\infty} \gamma^{t'-t} r(\mathbf{s}_{t'}) \Big\vert \mathbf{a}_{t}, \mathbf{a}_{t'} \!=\! \pi^{*}(\mathbf{s}_{t'}),\forall t'\!>\!t \right].
\end{equation}
Thus, the V function measures the long-term performance given the current state, while the Q function can be used for comparing the long-term performance affected by different current actions.

The relation between the two value functions is given by~\cite{sutton2018reinforcement}\vspace{-0.1cm}
\begin{equation}\label{eq:Bellman_Q}
    Q(\mathbf{s}_{t}, \mathbf{a}_{t}) 
    = r(\mathbf{s}_{t}) + \gamma \sum_{\mathbf{s}_{t+1}}\operatorname{Pr}(\mathbf{s}_{t+1}|\mathbf{s}_{t}, \mathbf{a}_{t}) V(\mathbf{s}_{t+1}). \vspace{-0.1cm}
\end{equation}
In particular, the Q function satisfies the Bellman equation:\vspace{-0.1cm}
\begin{equation}\label{eq:optimal Q}
     Q(\mathbf{s}_t, \mathbf{a}_t) = r(\mathbf{s}_t) + \mathbb{E}_{\mathbf{s}_{t+1}} \left[ \gamma \max_{\mathbf{a}_{t+1}}Q \left(\mathbf{s}_{t+1}, \mathbf{a}_{t+1}\right) \right].\vspace{-0.1cm}
\end{equation}

\subsection{DDPG for scheduling policy optimization}
A DDPG agent has two neural networks (NNs): an actor NN with the parameter set $\bm{\alpha}$ and a critic network with the parameter set $\bm{\beta}$. In particular, the actor NN aims to approximate the optimal policy $\pi^{*}(\mathbf{s})$ by $\pi(\mathbf{s}; \bm{\alpha})$ after well-training, while the critic NN approximates $Q(\mathbf{s}_t, \mathbf{a}_t)$ by $Q(\mathbf{s}_t, \mathbf{a}_t; \bm{\beta})$. 
The overall training process for the two NNs works iteratively based on sampled data named \emph{transitions}. Each transition $\mathcal{T}_{i} \triangleq (\mathbf{s}_i, \mathbf{a}_{i}, r_i, \mathbf{s}_{i+1})$,  consists of the state, action, reward, and next state.
To train the actor NN, the critic NN's Q values are utilized for characterizing the performance of the actor NN's policy over the sampled data; to train the critic NN, the transitions generated by the actor NN's policy are utilized to evaluate the accuracy of Q function approximation.
In particular, 
the approximation error given $\mathcal{T}_i$ is defined as the temporal difference (TD) error\vspace{-0.1cm}
\begin{equation}\label{eq:TD}
    \mathsf{TD}_{i}^{2} \triangleq \left( y_{i}-Q(\mathbf{s}_{i},\mathbf{a}_{i};\bm{\beta}) \right)^{2},\vspace{-0.1cm}
\end{equation}
where $y_{i} = r_{i} + \gamma \max_{{\mathbf{a}}_{i+1}} Q(\mathbf{s}_{i+1},{\mathbf{a}}_{i+1};\bm{\beta}^{-})$ is the estimation of the Q function at next step and $\bm{\beta}^{-}$ is the critic NN parameters from the previous iteration.
In other words, the training of the critic NN is to minimize the TD errors over the transitions, making the approximated Q function satisfy the Bellman equation~\eqref{eq:optimal Q}~\cite{sutton2018reinforcement}. 

\begin{remark}
    Our work focuses on the effective training of the critic NN by utilizing the partial monotonicity of the Q function discussed in the section below, and the training method for the actor NN is identical to the baseline DDPG. Furthermore, the proposed training methods can be applied to solve other DRL problems that involve monotonic critic NNs and can seamlessly integrate with actor-critic DRL algorithms like Twin Delayed DDPG (TD3) and Proximal Policy Optimization (PPO).
\end{remark}

\vspace{-0.3cm}
\section{Partial Monotonicity of Q Function} \label{sec: proof of monotonicity}

Before proceeding further, we need the following technical lemma related to the V function monotonicity.
\begin{lemma}[Monotonicity of V function w.r.t. AoI states~\cite{chen2022seDRL}]\label{lemma:monotone V}
    Consider states $\mathbf{s} = (\bm{\tau}, \mathbf{H})$ and $\mathbf{s}'_{\text{AoI}} = (\bm{\tau}'_{(n)}, \mathbf{H})$, where $\bm{\tau}'_{(n)} = (\tau_{1}, \dots, \tau'_{n}, \dots, \tau_{N})$ and $\tau'_{n} \geq \tau_{n}$. For any $n\in\{1,\dots,N\}$, the following holds \vspace{-0.1cm}
    \begin{equation}
        V(\mathbf{s}'_{\text{AoI}}) \leq V(\mathbf{s}).\vspace{-0.1cm}
    \end{equation}
\end{lemma}

Now we establish the partial monotonicity of the Q function in terms of the AoI and channel states as below.
\begin{theorem}[Monotonicity of Q function w.r.t. AoI states]\label{theo: AoI}
    Consider action $\mathbf{a} = (a_1, \dots, a_N)$, and states $\mathbf{s} = (\bm{\tau}, \mathbf{H})$ and $\mathbf{s}'_{\text{AoI}} = (\bm{\tau}'_{(n)}, \mathbf{H})$, where $\bm{\tau}'_{(n)} = (\tau_{1}, \, \dots, \, \tau'_{n}, \, \dots, \tau_{N})$, $\tau'_{n} \geq \tau_{n}$. The following holds\vspace{-0.1cm}
    \begin{align}
        \label{ineq: Q aoi}
        Q(\mathbf{s}'_{\text{AoI}}, \mathbf{a}) & \leq Q(\mathbf{s}, \mathbf{a}).\vspace{-0.2cm}
    \end{align}
\end{theorem}
\begin{proof}
For notation simplicity, we assume that $a_{n} = m, \forall m \in \{1,\dots,M\}$.
    From~\eqref{eq:q} and~\eqref{eq:tau}, we have 
    \begin{align}
        \!\!\!\!\!\!\!\!\operatorname{Pr}(\bm{\tau}^{+}|\bm{\tau}\!, \! \mathbf{H},\! \mathbf{a}) 
        & \!=\! \prod_{n=1}^{N} \operatorname{Pr} (\tau_{n}^{+}|\tau_{n}, \mathbf{h}_{n}, a_{n}) \\
        & \!\!=\! \operatorname{Pr} (\tau_{n}^{+}|\tau_{n}, \!\mathbf{h}_{n}, \! a_{n}) \!\operatorname{Pr} (\bm{\tau}_{\backslash n}^{+}|\bm{\tau}_{\backslash n}, \!\mathbf{H}_{\backslash n}, \!\mathbf{a}_{\backslash n}), \label{eq: transition tau}
    \end{align}
    where $\bm{\tau}_{\backslash n} = (\tau_{1}, \dots, \tau_{n-1}, \tau_{n+1}, \dots, \tau_{N})$ and $\mathbf{a}_{\backslash n} = (a_{1}, \dots, a_{n-1}, a_{n+1}, \dots, a_{N})$ represent the AoI states and the actions of all devices except device $n$, respectively;  $\mathbf{h}_{n} = (h_{n,1}, h_{n,2}, \dots, h_{n,M})$ is the vector channel state between device $n$ and the remote estimator and $\mathbf{H}_{\backslash n} = (\mathbf{h}_{1}, \dots, \mathbf{h}_{n-1}, \mathbf{h}_{n+1}, \dots,  \mathbf{h}_{N})$. By using~\eqref{eq: transition p}, ~\eqref{eq:Bellman_Q}, and~\eqref{eq: transition tau}, it can be derived that
    \begin{align}
        \!\!Q(\mathbf{s}, \mathbf{a}) 
        \!=\! r(\mathbf{s}) 
        + & \gamma \! \sum_{\mathbf{H}^{+}} \sum_{\bm{\tau}_{\backslash n}^{+}} \sum_{\tau_{n}^{+}} \operatorname{Pr} (\mathbf{H}^{+}) \operatorname{Pr} (\bm{\tau}_{\backslash n}^{+}|\bm{\tau}_{\backslash n}, \mathbf{H}_{\backslash n}, \mathbf{a}_{\backslash n}) \!\! \\ 
        & \label{eq: Q n -n} 
        \times \operatorname{Pr} (\tau_{n}^{+}|\tau_{n}, \mathbf{h}_{n}, a_{n}) V(\mathbf{s}^{+}). 
    \end{align}
    Since the state $\mathbf{s}'_{\text{AoI}}$ is identical to the state $\mathbf{s}$ except the device $n$'s AoI state $\tau_{n}$, 
    to prove~\eqref{ineq: Q aoi} from~\eqref{eq: Q n -n} and $r(\mathbf{s}'_{\text{AoI}}) \leq r(\mathbf{s})$, we only need to prove the following inequality
    \begin{align}
        & \sum_{{\tau_{n}'}^{+}} \operatorname{Pr} ({\tau_{n}'}^{+}|\tau_{n}, \mathbf{h}_{n}, a_{n}) V({\mathbf{s}'_{\text{AoI}}}^{+}) \\
        \label{ineq:Pr x V}
        & \leq \sum_{{\tau_{n}}^{+}} \operatorname{Pr} ({\tau_{n}}^{+}|\tau_{n}, \mathbf{h}_{n}, a_{n}) V({\mathbf{s}}^{+}).
    \end{align}
    From~\eqref{eq:tau}, each of $\mathbf{s}^{+}$ and ${\mathbf{s}'_{\text{AoI}}}^{+}$  has two different states in~\eqref{ineq:Pr x V} depending on the different transition of device $n$'s AoI state. Then, by taking~\eqref{eq: transition p 2} into \eqref{ineq:Pr x V}, the latter is equivalent to
    \begin{align}
        & \!\!\! \!(1\!- p_{n,m}) V\!(1, \! \bm{\tau}_{\backslash n}^{+}, \! {\mathbf{H}}^{+}) + p_{n,m} V\!({\tau_{n}'} \!+\! 1, \! \bm{\tau}_{\backslash n}^{+}, \! {\mathbf{H}}^{+})\\
        \label{ineq: pV' < pV aoi}
        & \!\!\! \leq (1\!- p_{n,m}) V\!(1, \! \bm{\tau}_{\backslash n}^{+}, \! \mathbf{H}^{+}) \!+\! p_{n,m} V\!({\tau_{n}} \!+\! 1, \! \bm{\tau}_{\backslash n}^{+}, \! \mathbf{H}^{+}).
    \end{align}
    By using Lemma~\ref{lemma:monotone V}, we have $V\!({\tau_{n}'} \!+\! 1, \! \bm{\tau}_{\backslash n}^{+}, \! {\mathbf{H}}^{+}) \leq V\!({\tau_{n}} \!+\! 1, \! \bm{\tau}_{\backslash n}^{+}, \! \mathbf{H}^{+})$. The inequality~\eqref{ineq: pV' < pV aoi} holds directly.
\end{proof}

\begin{theorem}[Monotonicity of Q function w.r.t. channel states]\label{theo: channel}
    Consider action $\mathbf{a} = (a_1, \dots, a_N)$, and states $\mathbf{s} = (\bm{\tau}, \mathbf{H})$ and $\mathbf{s}'_{\text{Ch}} = (\bm{\tau}, \mathbf{H}'_{n,m})$ where $\mathbf{H}'_{n,m}$ is identical to $\mathbf{H}$ except the device-$n$-channel-$m$ state with $h'_{n,m} > h_{n,m}$. 
    \begin{itemize}
    \item[(i)] If $a_{n}=m$, then the following holds\vspace{-0.1cm}
    \begin{align}
        \label{ineq: Q channel}
        Q(\mathbf{s}'_{\text{Ch}}, \mathbf{a}) \leq Q(\mathbf{s}, \mathbf{a}).\vspace{-0.1cm}
    \end{align}
    \item[(ii)] If $a_{n} \neq m$, then the equality in~\eqref{ineq: Q channel} holds, i.e., we have \vspace{-0.1cm}
    \begin{align}
        \label{eq: Q channel}
        Q(\mathbf{s}'_{\text{Ch}}, \mathbf{a}) = Q(\mathbf{s}, \mathbf{a}).\vspace{-0.1cm}
    \end{align}
    \end{itemize}
\end{theorem}
\begin{proof}
    (i) Based on~\eqref{eq: Q n -n}, we only need to prove that
    \begin{align}
         & \sum_{{\mathbf{H}'}^{+}} \sum_{{\tau_{n}}^{+}} \operatorname{Pr} ({\mathbf{H}'}^{+}) \operatorname{Pr} ({\tau_{n}}^{+}|\tau_{n}, \mathbf{h}'_{n}, a_{n}) V({\mathbf{s}'_{\text{Ch}}}^{+}) \\
        \label{ineq:Pr x V 2}
        & \leq \sum_{{\mathbf{H}}^{+}} \sum_{{\tau_{n}}^{+}} \operatorname{Pr} (\mathbf{H}^{+}) \operatorname{Pr} ({\tau_{n}}^{+}|\tau_{n}, \mathbf{h}_{n}, a_{n}) V({\mathbf{s}}^{+}),
    \end{align}
    which is equivalent to
    \begin{align}
        &  (1-p'_{n,m}) \sum_{{\mathbf{H}'}^{+}}\operatorname{Pr} ({\mathbf{H}'}^{+}) V(1,  \bm{\tau}_{\backslash n}^{+},  {\mathbf{H}'}^{+}) \\
        & + p'_{n,m} \sum_{{\mathbf{H}'}^{+}}\operatorname{Pr} ({\mathbf{H}'}^{+}) V({\tau_{n}} + 1, \bm{\tau}_{\backslash n}^{+},  {\mathbf{H}'}^{+})  \\
        & \leq (1-p_{n,m}) \sum_{{\mathbf{H}}^{+}}\operatorname{Pr} ({\mathbf{H}}^{+}) V(1,  \bm{\tau}_{\backslash n}^{+},  {\mathbf{H}}^{+}) \\
        \label{ineq: pV' < pV channel}
        & + p_{n,m} \sum_{{\mathbf{H}}^{+}}\operatorname{Pr} ({\mathbf{H}}^{+}) V({\tau_{n}} + 1,  \bm{\tau}_{\backslash n}^{+},  {\mathbf{H}}^{+}).
    \end{align}
    From~\eqref{eq:q}, we have 
    \begin{equation}\label{eq: Pr(H')V' = Pr(H)V}
        \sum_{{\mathbf{H}'}^{+}} \operatorname{Pr}({\mathbf{H}'}^{+}) V( \bm{\tau}^{+}, {\mathbf{H}'}^{+}) 
        = \sum_{{\mathbf{H}}^{+}} \operatorname{Pr}({\mathbf{H}}^{+}) V(\bm{\tau}^{+}, {\mathbf{H}}^{+}),
    \end{equation}
    where $\bm{\tau}^{+} = (1,  \bm{\tau}_{\backslash n}^{+})$ or $(\tau_{n} + 1,  \bm{\tau}_{\backslash n}^{+})$.
    By using~\eqref{eq: Pr(H')V' = Pr(H)V}, the inequality~\eqref{ineq: pV' < pV channel} is simplified as
    \begin{align}
        (p'_{n,m} \!-\! p_{n,m}) \!\sum_{{\mathbf{H}}^{+}} \operatorname{Pr} ({\mathbf{H}}^{+}) \Big[ & V({\tau_{n}} \!+\! 1, \! \bm{\tau}_{\backslash n}^{+},  \!{\mathbf{H}}^{+}) \\
        \label{ineq: pV' < pV channel 2}
        & - V(1, \! \bm{\tau}_{\backslash n}^{+}, \! {\mathbf{H}}^{+})\Big] \!\leq\! 0.
    \end{align}
    Due to the facts that $V\!({\tau_{n}} \!+\! 1, \! \bm{\tau}_{\backslash n}^{+}, \! {\mathbf{H}}^{+}) \!\leq\! V\!(1, \! \bm{\tau}_{\backslash n}^{+}, \! {\mathbf{H}}^{+})$ and $p'_{n,m} \geq p_{n,m}$, \eqref{ineq: pV' < pV channel 2} holds directly.

    (ii) Let $a_{n} = \bar{m} \neq m$. Similar to the proof of (i), after simplifications, we only need to prove the following equality
    \begin{align}
        (p_{n,\bar{m}} \!-\! p_{n,\bar{m}}) \!\sum_{{\mathbf{H}}^{+}} \operatorname{Pr} ({\mathbf{H}}^{+}) \Big[ & V({\tau_{n}} \!+\! 1, \! \bm{\tau}_{\backslash n}^{+},  \!{\mathbf{H}}^{+}) \\
        & - V(1, \! \bm{\tau}_{\backslash n}^{+}, \! {\mathbf{H}}^{+})\Big] = 0,
    \end{align}
    which is directly valid.
\end{proof}

\vspace{-0.3cm}
\section{Partially Monotonic Critic Neural Network}
Building on the partial monotonicity of the Q function, we aim to develop novel training methods guaranteeing the monotonicity of the critic NN that can help solve Problem~\ref{pro1} more effectively compared to conventional learning methods.
We consider two partially monotonic critic NN training methods: (i) network architecture-enabled monotonicity and (ii) regularization-based monotonicity in the sequel.

\vspace{-0.3cm}
\subsection{Network Architecture-enabled Monotonicity} \label{sec: MA-DDPG}
We propose a three-layer critic NN architecture (a shallow NN) that directly guarantees the partial monotonicity of the Q function, as shown in Fig. 2 (a).
The critic NN integrates the outputs of two NN processing the state and the action inputs separately.
In particular, the state-processing part (on the top of Fig. 2(a)) has to be a monotonic decreasing function.
We adopt the monotonic NN in~\cite{daniels2010monotone}, which guarantees that the output is a monotonic decreasing function of the input with the following conditions: 1) the activation functions are non-decreasing (e.g., sigmoid or ReLU), and 2) the weights satisfy
\begin{equation}\label{eq: network constraint}
    \beta^{(1, 2)}_{\bar{i}, \bar{j}} \beta^{(2, 3)}_{\bar{j}} \leq 0, \forall \bar{i}\in\{1,N(M+1)\}, \bar{j} \in \{1,\dots, U\}.
\end{equation}
In~\eqref{eq: network constraint}, $\beta^{(1, 2)}_{\bar{i}, \bar{j}}$ denotes the weight for the connection between $\bar{i}$-th input node and the $\bar{j}$-th hidden node, $\beta^{(2, 3)}_{\bar{j}}$ denotes the weight for the connection between the $\bar{j}$-th hidden node and the output node in the critic NN, and $U$ is the size of the hidden layer.
There is no specific requirement on the parameters of the action-processing part (at the bottom of Fig. 2(a)).

First, we adopt the sigmoid activation function with positive output. Then, to satisfy the constraint~\eqref{eq: network constraint} and to ensure negative Q value outputs (see the definition in \eqref{eq: def Q}) during training, we restrict the weights $\beta^{(1, 2)}_{\bar{i}, \bar{j}}$ and $\beta^{(2, 3)}_{\bar{j}}$ to be non-negative and non-positive, respectively, i.e., $\beta^{(1, 2)}_{\bar{i}, \bar{j}} \leftarrow \max (0, \beta^{(1, 2)}_{\bar{i}, \bar{j}} )$ and $\beta^{(2, 3)}_{\bar{i}, \bar{j}} \leftarrow \min (0, \beta^{(2,3)}_{\bar{i}, \bar{j}} )$.

The advantage of the presented approach is that the monotonicity is strictly and automatically ensured by the NN architecture during the training at each time step. A potential drawback is that the ``shallow" monotonic critic NN has a single hidden layer, limiting its capability for approximating complex functions, which are required when solving scheduling problems with large state and action spaces.
Note that although there are some deep monotonic NN, such as deep lattice networks \cite{you2017deeplattiice}, these NNs have very complex architectures and are computationally expensive for large inputs. 
Therefore, we propose the following approach for training a ``deep" monotonic critic NN, which is a multi-layer fully connected NN without a specific NN architecture requirement.

\begin{figure}[t]
    \centering
    \vspace{-0.1cm}
    \begin{subfigure}[]
        {
        \begin{minipage}[t]{0.47\linewidth}
        \centering
        \includegraphics[width=\textwidth]{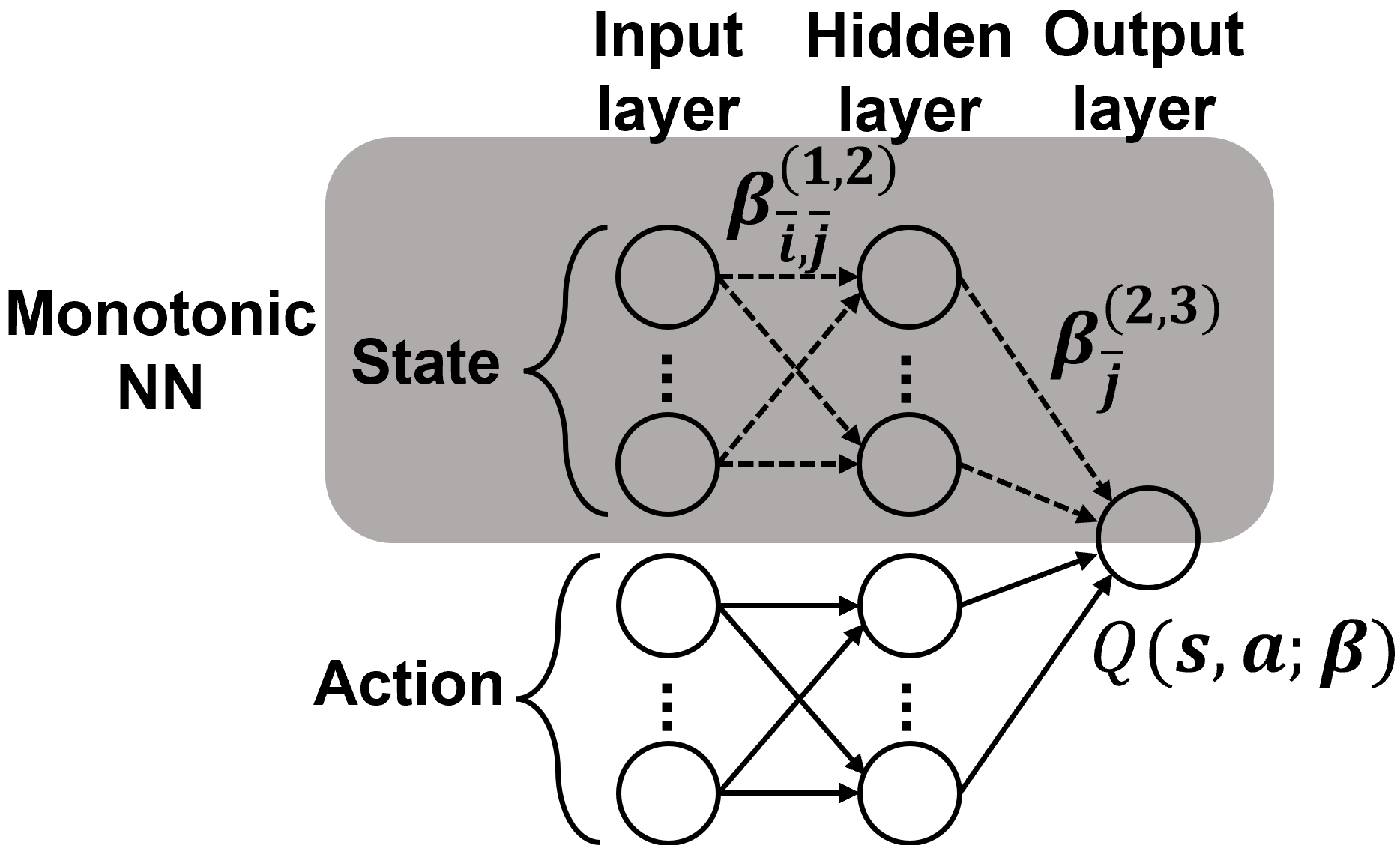}
        \end{minipage}
        \label{fig: monotone critic network}
        }
    \end{subfigure}
    \hspace{-0.25cm}
    \begin{subfigure}[]
        {
        \begin{minipage}[t]{0.455\linewidth}
        \centering
        \includegraphics[width=1\textwidth]{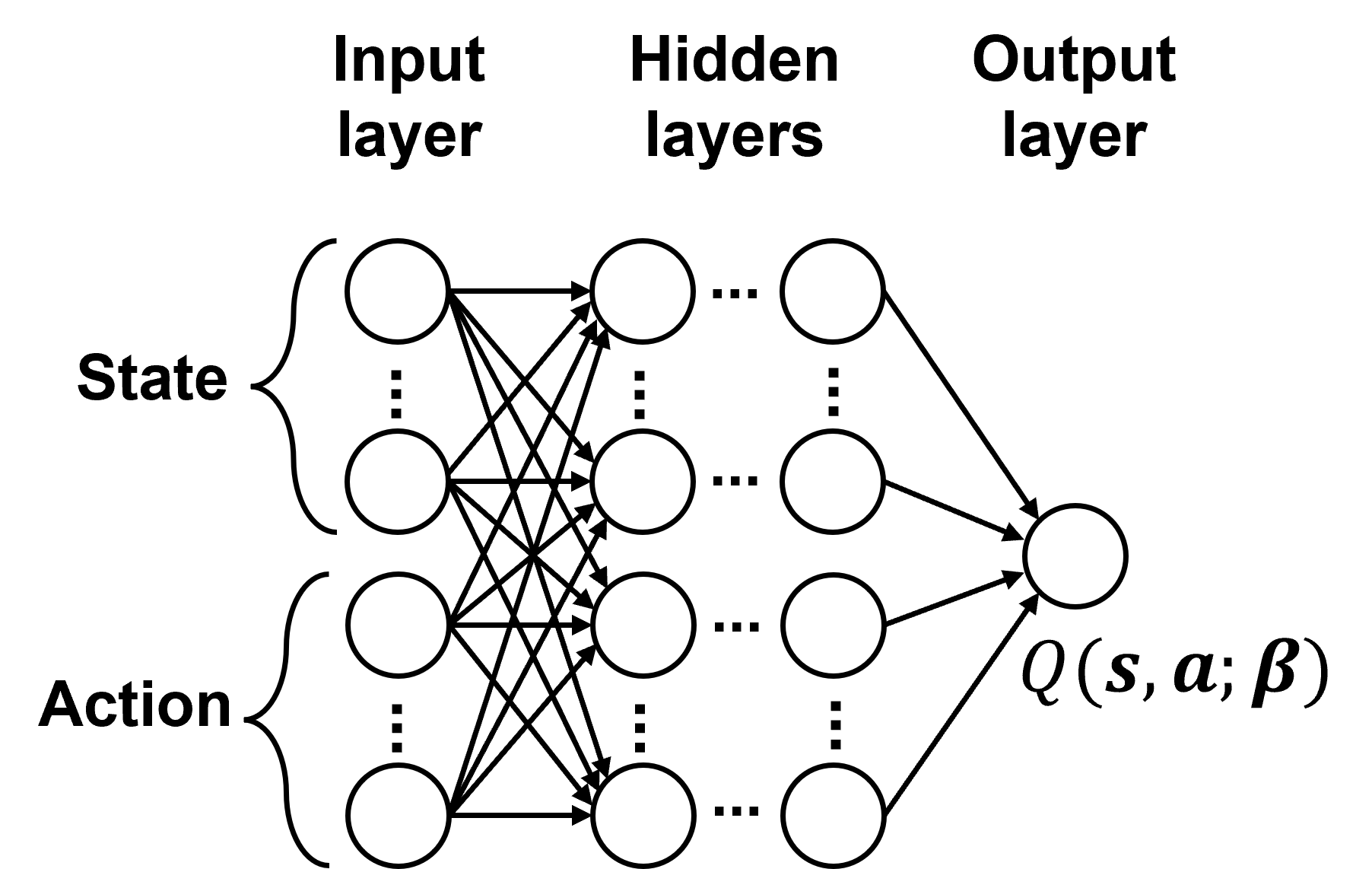}
        \end{minipage}
        \label{fig: general critic network}
        }
    \end{subfigure}
    \vspace{-0.4cm}
    \caption{Partially monotonic critic NNs. {(a) A shallow critic NN with network architecture-enabled monotonicity, where the NN has only one hidden layer. (b) A deep critic NN with regularization-based monotonicity.}
    }
    \vspace{-0.5cm}
    \label{fig: critic network}
\end{figure}

\subsection{Regularization-based Monotonicity}
We introduce penalty terms in the loss function of the critic NN training to enforce a monotonically \emph{decreasing} relationship between the state input and the output.
In particular, we propose two types of penalty for monotonicity regularization.

\subsubsection{Positive derivative penalty (type I)}
Let $s_{i,j}$ denote the $j$-th state of the length-$N(M+1)$ state vector $\mathbf{s}_i$, which can be an AoI state (i.e., $j\leq N$) or a channel state (i.e., $j> N$).
The partial derivative of the Q function $Q(\mathbf{s}_i, \mathbf{a}_i; \bm{\beta})$ at the state input $s_{i,j}$ is given as
\begin{equation}\label{eq: derivation}
    \mathsf{QD}_{i,j} = \nabla_{s_{i,j}} Q(\mathbf{s}_i, \mathbf{a}_i; \bm{\beta}) \quad  \forall j \in \{ 1, \dots, N (M+1)\}.
\end{equation}
To enforce the critic NN approximating a monotonically decreasing function at $s_{i,j}$,  the corresponding positive derivative penalty is $\max(0,\mathsf{QD}_{i,j})$.

The computation of the derivative \eqref{eq: derivation} requires backward propagation, which has a higher computational complexity than forward propagation, i.e., evaluating a Q value by feeding the critic NN with an input. Thus, we propose a forward-propagation-only penalty scheme below.
\subsubsection{Positive increment penalty (type II)}
We introduce the Q value increment from states $\mathbf{s}_{i}$ to $\mathbf{s}'_{i, j}$ as
\begin{equation}\label{eq: MD error}
    \mathsf{QI}_{i,j} = Q(\mathbf{s}'_{i,j}, \mathbf{a}_{i}; \bm{\beta}) - Q(\mathbf{s}_{i}, \mathbf{a}_{i}; \bm{\beta}),
\end{equation}
where $\mathbf{s}'_{i, j}$ is identical to $\mathbf{s}_{i}$ except for the $j$-th state with one-step increment, $s'_{i, j} = s_{i, j} + 1$.
To enforce the critic NN approximating a monotonically decreasing function from $\mathbf{s}_{i}$ to $\mathbf{s}'_{i, j}$, we define the positive increment penalty as $\max(0,\mathsf{QI}_{i,j})$.

Potentially, given the vector state $\mathbf{s}_i$ and based on the penalty defined above related to $s_{i,j}$, one can calculate penalty terms for all $j \in \{1,\dots, N(M+1)\}$, leading to high computational cost when $N$ and $M$ are large. Therefore, we propose a novel scheme to compute the penalty effectively.

\emph{\underline{Stochastic penalty sampling scheme.}}
First, given the state action pair $(\mathbf{s}_i,\mathbf{a}_i)$, we define an effective set $\mathcal{J}_i \subset \{1,\dots,N(M+1)\}$ of the individual states for penalty calculation.
From \eqref{eq: derivation} and \eqref{eq: MD error}, a penalty term is effective if $\mathsf{QD}_{i,j}$ or $\mathsf{QI}_{i,j}$ is positive and the Q value changes dramatically. From \eqref{eq: Q channel} of Theorem~\ref{theo: channel}, channel $h_{n,m}$'s quality improvement does not change the Q value at all, if device $n$ is not scheduled at channel $m$.
Thus, we can eliminate those cases when calculating the penalty.
We define an indicator on whether a channel state $s_{i,j}$ with $ j > N$ should be considered for penalty evaluation or not, i.e.,
\begin{equation}
    \!\!\!\!I(s_{i,j}) = \begin{cases}
        j & \text{\!$\mathbf{a}_i$ utilizes $s_{i,j}$'s corresponding channel,} \\
        \O & \text{\!otherwise.}
    \end{cases}\!\!
\end{equation}
Then, the effective set for penalty calculation is
\begin{equation}\label{eq: J set}
    \mathcal{J}_{i} \!=\! \{ 1, \dots, N, I(s_{i,N+1}), \dots, I(s_{i,N(M+1)}) \}.
\end{equation}
Second, we randomly generate a uniformly down-sampled set $\bar{\mathcal{J}}_i$ with size $K \ll N(M+1)$ based on $\mathcal{J}_i$.
From~\eqref{eq: derivation} and \eqref{eq: MD error}, the loss function of the transition $\mathcal{T}_i$ regularized by the penalty terms determined by $\bar{\mathcal{J}}_i$ is
\begin{equation}
    L_{i} (\bm{\beta}) \!=\!\!
    \begin{cases}
        \!\mathsf{TD}_{i}^{2} \!+\! \sum_{j \in \bar{\mathcal{J}}_i} \max \left( 0, \mathsf{QD}_{i,j} \right) & {\text{type I regularized}}, \\
        \!\mathsf{TD}_{i}^{2} \!+\! \sum_{j \in \bar{\mathcal{J}}_i} \max \left(0, \mathsf{QI}_{i,j}\right) & {\text{type II regularized}}.
    \end{cases}
\end{equation}
A well-trained critic NN $\bm \beta$ should minimize both the TD error and the penalty term in the loss function above.

\begin{table}[t]
    \footnotesize
    \setlength\tabcolsep{6pt}
    \centering
    \vspace{-0.5cm}
    \caption{Summary of Hyperparameters of different DDPG algorithms}
    \vspace{-0.3cm}
    \label{tab:Setup}
    \begin{tabular}{c|c}
         \hline \hline
         \textbf{Hyperparameters} & Value \\
         \hline
         \rowcolor[HTML]{EFEFEF} 
         Mini-batch size, $B$                       & 128 \\
         Time horizon for each episode              & 500 \\
         \rowcolor[HTML]{EFEFEF} 
         Experience replay memory size, $K$         & 20000 \\
         Learning rate of actor network and critic network       & 0.0001 and 0.001 \\
         \rowcolor[HTML]{EFEFEF}
         Decay rate of learning rate                & 0.001 \\
         Soft parameter for target update, $\delta$ & 0.005 \\
         \rowcolor[HTML]{EFEFEF} 
         Input dimension of actor network   & $N+N \times M$ \\
         Output dimension of actor network   & $N$ \\
         \rowcolor[HTML]{EFEFEF} 
         Input dimension of critic network   & $2N+N \times M$ \\
         Output dimension of critic network   & $1$ \\
         \hline \hline
    \end{tabular}
\vspace{-0.6cm}
\end{table}

\vspace{-0.3cm}
\section{Numerical Experiments} \label{sec: simulation}
In this section, we evaluate and compare the performance of the three monotonicity-enforced training methods for DDPG proposed in the previous section, namely monotonic architecture-based DDPG (MA-DDPG), type I monotonicity-regularized DDPG (MRI-DDPG), and type II monotonicity-regularized DDPG (MRII-DDPG).

The computing platform for running the numerical experiments has an Intel Core i5 9400F CPU @ 2.9 GHz, 16GB RAM, and an NVIDIA RTX 2070 GPU.
The experiments are built on the remote estimation system described in Example~\ref{eg:1}. 
In terms of the system parameters, each process has a two-dimensional state and scalar measurement, i.e., $l_{n} = 2$ and $c_{n} = 1$. The spectral radius of $\mathbf{A}_{n}$ and the entries of $\mathbf{C}_{n}$ are randomly sampled from the range (1, 1.3) and (0,1), respectively. The channel states are quantized into $\bar{h}=5$ levels from the Rayleigh fading channel with the scale parameter drawn uniformly from the range (0.5, 2)~\cite{chen2022seDRL}. The packet drop probabilities of different levels are set as 0.2, 0.15, 0.1, 0.05, and 0.01.
The critic NN for a 6-sensor-3-channel system and other cases has 1 hidden layer and 3 hidden layers, respectively, each with 1024 nodes. The actor NN for all cases has 3 hidden layers with 1024 nodes. All the NNs are fully connected NNs.
The number of penalty terms evaluated is $K=2$. The setting of the other hyperparameters for the different DDPG algorithms are shown in Table~\ref{tab:Setup}.

\begin{table*}[t]
	\footnotesize
	\setlength\tabcolsep{7.5pt}
	\centering
	\caption{Performance Comparison of the Different DDPG Algorithms}
	\label{tab:test result}
	\setlength{\tabcolsep}{1.5mm}
	\vspace{-0.3cm}
	\begin{tabular}{c|c|c|c|c|c|c|c|c|c|c|c|c}
		\hline \hline
		\multirow{2}*{\thead{System Scale \\ $(N,M)$}} & \multirow{2}*{\thead{Algorithms}} & \multirow{2}*{\thead{Training time/episode \\ (second)}} & \multicolumn{2}{c|}{Experiment 1} & \multicolumn{2}{c|}{Experiment 2} & \multicolumn{2}{c|}{Experiment 3} & \multicolumn{2}{c|}{Experiment 4} & \multicolumn{2}{c}{Experiment 5}\\
          \cline{4-13}
                             &  & &  \thead{NEC} & MSE & \thead{NEC} & MSE &\thead{NEC} & MSE & \thead{NEC} & MSE & \thead{NEC} & MSE\\
		\hline
		\rowcolor[HTML]{EFEFEF}
		$(6, 3)$  & DDPG     & 3.10                    &   $-$    & $-$&            $-$     & $-$              & 107 & 133.7559 & $-$ & $-$ & $-$ & $-$\\
		       &\textbf{MA-DDPG}  & 3.17                    &  \textbf{58}      &\textbf{50.4761} &   \textbf{80} & \textbf{80.3567} & \textbf{30} & \textbf{56.3479}   & \textbf{56} & \textbf{66.5171} & \textbf{83} & \textbf{74.1661}\\ 
          	 \rowcolor[HTML]{EFEFEF}
        		 & MRI-DDPG & 4.69   &    74    & 57.3178 &          $-$   & $-$              & 82 & 61.4528  & $-$ & $-$ & $-$ & $-$\\
        		 & MRII-DDPG & 4.01  &     66   & 57.4541 &          $-$   & $-$              & 53 & 61.4805  & $-$ & $-$ & $-$ & $-$\\
		\hline
		
		\rowcolor[HTML]{EFEFEF}
		$(14, 7)$ & DDPG     & 6.16  &   189     &  160.6451      &  162   & 136.5146    & $-$  & $-$  & $-$ & $-$ & 132 & 127.8284\\
          	 
        		 & MRI-DDPG & 8.15  &   148     &  122.1874      &  105   &  114.3175   &  204 & 140.6513 & 176 & 207.1918 & 97 & 110.1273 \\
                \rowcolor[HTML]{EFEFEF}
        		 & \textbf{MRII-DDPG} & 7.20  & \textbf{87} &  \textbf{112.3492}  &  \textbf{60} &  \textbf{108.8440} &  \textbf{160} & \textbf{131.0342}  & \textbf{113}& \textbf{187.2761} & \textbf{68} & \textbf{103.4623}\\
		\hline
		
		$(20, 10)$  & DDPG     & 8.32  &  193    & 210.6116    & 188  & 199.5522  & 213 & 175.2343 & $-$ & $-$ & $-$ & $-$\\
          	  \rowcolor[HTML]{EFEFEF}
        		  & MRI-DDPG & 9.83  &   158   &  172.7708   &  163 &  181.6881 & 182 & 144.6849  & 127 & 180.3464 & 273 & 216.3028\\
        		  & \textbf{MRII-DDPG} & 9.37  &   \textbf{116}   &  \textbf{155.7731}   & \textbf{109}    &   \textbf{170.6473}  & \textbf{131} &\textbf{134.4982} & \textbf{103} &\textbf{167.6880} & \textbf{201} & \textbf{181.9388} \\
		\hline \hline
	\end{tabular}
 \vspace{-0.4cm}
\end{table*}

Figs.~\ref{fig:training curve 6-3} and~\ref{fig:training curve 14-7} show the average sum MSE of all processes during the training achieved by different DDPG algorithms with $N=6$, $M=3$ and $N=14$, $M=7$, respectively.
Fig.~\ref{fig:training curve 6-3} illustrates that the MA-DDPG reduces the average sum MSE by $15\%$ when compared to the MRI-DDPG and MRII-DDPG, while the baseline DDPG cannot converge with the shallow critic NN. Thus, for a small-scale system, MA-DDPG is the most effective algorithm.
Fig.~\ref{fig:training curve 14-7} shows that MRI-DDPG and MRII-DDPG save about $15\%$ and $50\%$ training episodes for convergence, respectively, and also reduce the average sum MSE by $30\%$ when compared to DDPG. MRII-DDPG performs better than MRI-DDPG, since the state space is discrete, and MRII-DDPG directly regularizes the Q function monotonicity at the discrete state space.
Note that MA-DDPG cannot converge due to its shallow NN architecture.

\begin{figure}[t]
    \centering
    \includegraphics[width=0.8\linewidth]{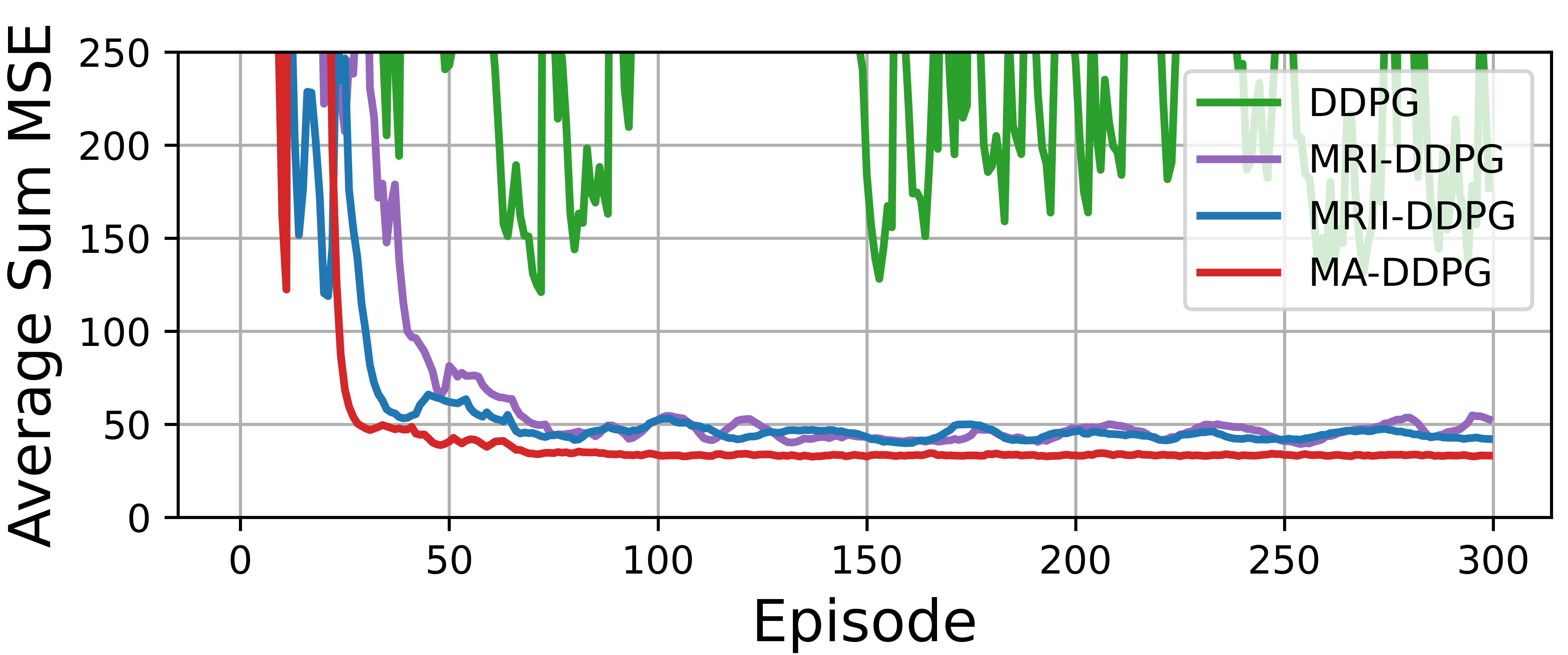}
    \vspace{-0.35cm}
    \caption{Average sum MSE during training with $N = 6, M = 3$. The critic NN has only one hidden layer -- a shallow NN setup.}
    \vspace{-0.5cm}
    \label{fig:training curve 6-3}
\end{figure}

\begin{figure}[t]
    \centering
    \includegraphics[width=0.8\linewidth]{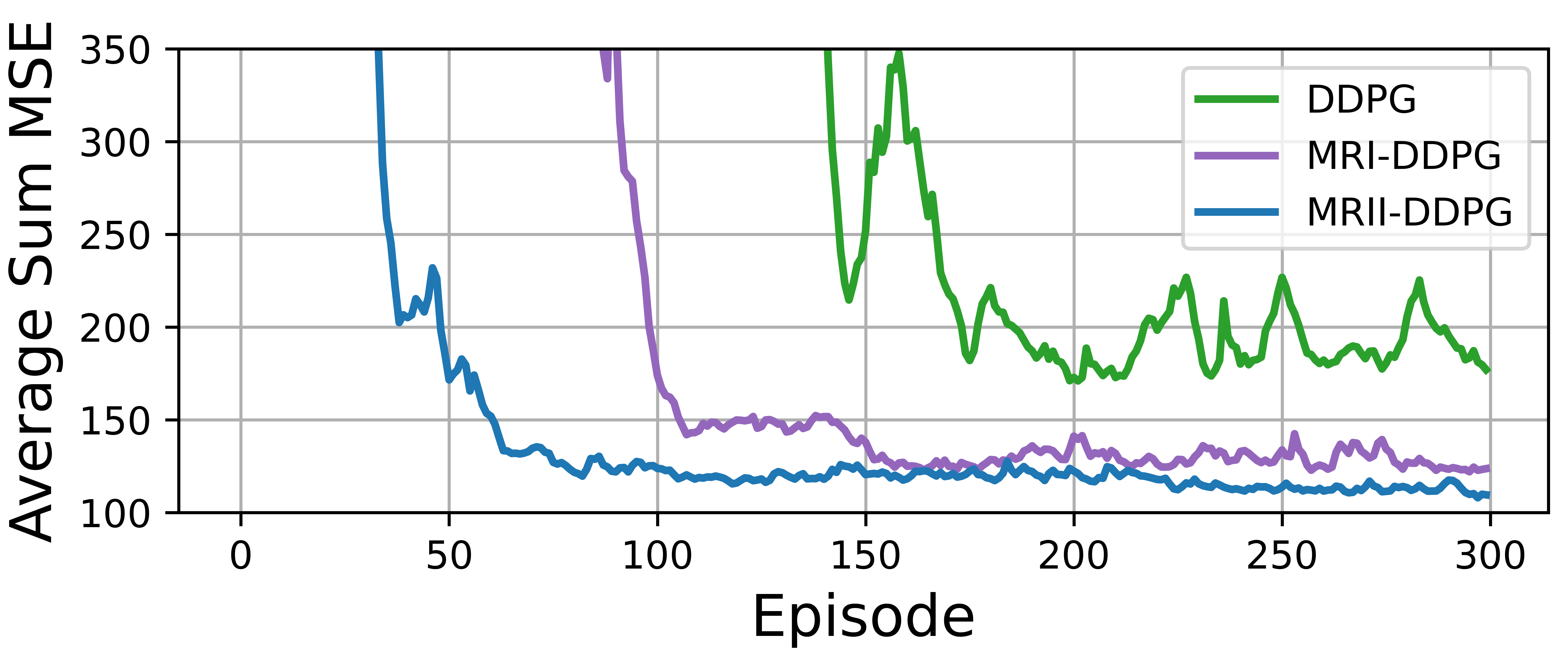}
    \vspace{-0.35cm}
    \caption{Average sum MSE during training with $N = 14,  M=7$. The critic NN has three hidden layers -- a deep NN setup.}
    \vspace{-0.6cm}
    \label{fig:training curve 14-7}
\end{figure}

In Table~\ref{tab:test result}, we evaluate the training time, the number of episodes for convergence (NEC), and the performance of different DDPG algorithms (based on 20000-step simulations) over different system scales and parameters.
We see that for the 20-sensor-10-channel systems, MRII-DDPG can reduce the average MSE by $27\%$ when compared to DDPG (at the cost of the increased training time per episode by $13\%$), whilst saving more than $30\%$ time for convergence (i.e., training time/episode $\times$ NEC), despite the extended training duration of each episode. In particular, MRII-DDPG solves the scheduling problem in Experiments 4 and 5 effectively, while the baseline DDPG fails to converge.

\section{Conclusion}
In this work, building on the monotonicity of the Q function of the optimal semantic-aware  scheduling policy, we have developed monotonicity-driven DRL algorithms to solve the scheduling problems effectively and reduce the MSE by about $30\%$ compared to the conventional DRL algorithm.
For future work, we will investigate semantic-aware resource allocation problems over practical time-correlated fading channels.

\vspace{-0.2cm}

\bibliographystyle{IEEEtran}

\end{document}